\documentclass[10pt,twocolumn,letterpaper]{article}

\usepackage{iccv}
\usepackage{times}
\usepackage{epsfig}
\usepackage{graphicx}
\usepackage{amsmath}
\usepackage{amssymb}
\usepackage{float}
\usepackage{booktabs} 
\usepackage{mathtools}
\usepackage{amsthm}
\usepackage{subfigure}
\RequirePackage{algorithm, algorithmic}
\usepackage{colortbl}
\usepackage{diagbox}
\usepackage{multicol}
\usepackage{multirow}
\theoremstyle{plain}
\newtheorem{theorem}{Theorem}[section]

\theoremstyle{definition}
\newtheorem{definition}[theorem]{Definition}

\theoremstyle{remark}

\newcommand{\norm}[1]{{\lVert {#1} \rVert}}

\usepackage[pagebackref=true,breaklinks=true,letterpaper=true,colorlinks,bookmarks=false]{hyperref}

\iccvfinalcopy 

\def\iccvPaperID{8852} 

\ificcvfinal\fi

\begin{document}

\title{DPM-OT: A New Diffusion Probabilistic Model Based on Optimal Transport}

\author{Zezeng Li$^{1,2}$, ~ShengHao Li$^{1}$, ~Zhanpeng Wang$^{3}$, ~Na Lei$^{1}$\thanks{Corresponding author: \emph{Na Lei (nalei@dlut.edu.cn)}}, ~Zhongxuan Luo$^{1}$, ~Xianfeng Gu$^4$\\ 	
\normalsize{$^1$School of Software, Dalian University of Technology, China}\\
\normalsize{$^2$Beijing Key Laboratory of Light-field Imaging and Digital Geometry, Capital Normal University, China}\\
\normalsize{$^3$School of Mathematical Sciences, University of the Chinese Academy of Sciences, China} \\
\normalsize{$^4$Computer Science and Applied Mathematics, State University of New York at Stony Brook, USA}
}

\maketitle
\ificcvfinal\thispagestyle{empty}\fi

\begin{abstract}
   Sampling from diffusion probabilistic models (DPMs) can be viewed as a piecewise distribution transformation, which generally requires hundreds or thousands  of steps of the inverse diffusion trajectory to get a high-quality image. Recent progress in designing fast samplers for DPMs achieves a trade-off between sampling speed and sample quality by knowledge distillation or adjusting the variance schedule or the denoising equation. However, it can't be optimal in both aspects and often suffer from mode mixture in short steps. To tackle this problem, we innovatively regard inverse diffusion as an optimal transport (OT) problem between latents at different stages and propose the \textbf{DPM-OT}, a unified learning framework for fast DPMs with a direct expressway represented by OT map, which can generate high-quality samples within around 10 function evaluations. By calculating the semi-discrete optimal transport map between the data latents and the white noise, we obtain an expressway from the prior distribution to the data distribution, while significantly alleviating the problem of mode mixture. In addition, we give the error bound of the proposed method, which theoretically guarantees the stability of the algorithm. Extensive experiments validate the effectiveness and advantages of \textbf{DPM-OT} in terms of speed and quality (FID and mode mixture), thus representing an efficient solution for generative modeling. Source codes are available at \url{https://github.com/cognaclee/DPM-OT}
\end{abstract}

\section{Introduction}
Diffusion probabilistic models (DPMs)~\cite{sohl2015deep,ho2020denoising,song2020score} are a class of new prevailing generative models which use a parameterized Markov chain to produce samples matching the data distribution after a finite time. Transitions of this chain include two processes: the diffusion process gradually adds noise to a data distribution and the sampling process gradually reverses each step of the noise corruption over a long trajectory of timesteps. DPMs are able to produce high-quality samples and even superior to the current SOTAs generative adversarial networks (GANs)~\cite{goodfellow2014generative} on many tasks, such as image generation~\cite{dhariwal2021diffusion,meng2021sdedit,choi2021ilvr}, video generation~\cite{ho2022video}, text-to-image generation~\cite{ramesh2022hierarchical}, point cloud generation~\cite{luo2021diffusion,nichol2022point}, shape generation~\cite{zhou20213d,zeng2022lion} and speech synthesis~\cite{chen2020wavegrad,chen2021wavegrad}. Despite their success, the sampling of DPMs often requires iterating over thousands of timesteps, which is two or three orders of magnitude slower~\cite{song2020denoising,bao2022analytic} than single-step generative models GANs and VAEs~\cite{kingma2013auto}.  

To accelerate the sampling process, the community has been focusing on fast DPMs. Existing works have successfully accelerated DPMs by knowledge distillation~\cite{salimans2022progressive,luhman2021knowledge}, or adjusting the variance schedule~\cite{san2021noise,nichol2021improved,lam2021bilateral,watson2021learning} or the denoising equation~\cite{song2020denoising,jolicoeur2021gotta,tachibana2021taylor,popov2021diffusion,liu2022pseudo,bao2022analytic,lu2022dpm,zhao2023unipc}. However, as~\cite{kong2021fast,liu2022pseudo} remarks, early fast samplers cannot maintain the quality of samples and even introduce new noise at a high speedup rate, which limits their practicability. Moreover, existing methods try to approximate a continuous diffusion process with a deep neural network, but ignore the discontinuity of the target data manifold at the class boundary, which leads to mode mixture in the generated images.

To resolve the above issues, we cast the denoising process as an OT problem and then compute the Brenier potential~\cite{brenier1987polar,brenier1991polar} to represent the OT map which is discontinuous at singularity sets~\cite{figalli2010regularity,chen2017partial,an2019ae} and thus avoids mode mixture. Then we construct an optimal trajectory between different timestep latents, which combines multiple denoising processes into an OT map, thus greatly shortening the sampling trajectory. Building upon it, we propose \textbf{DPM-OT} which can generate high-quality images within around 10 steps of inverse diffusion. In summary, our main contributions are:
\begin{itemize}
\item By combining OT and diffusion model, we propose a unified learning framework \textbf{DPM-OT} for fast DPMs.
\item \textbf{DPM-OT} computes the Brenier potential to represent the OT map between different timesteps latents which relieves mode mixture significantly.
\item We theoretically analyze the single-step error and give the upper bound of the error between the generated data distribution and the target data distribution.
\item Extensive experiments demonstrate \textbf{DPM-OT} outperforms SOTAs in quality, especially for mode mixture.
\end{itemize}

\section{Preliminaries}
In this paper, we are committed to providing a plug-and-play fast DPM framework by incorporating the OT into DPM. So, in this section, we first review the generalized DPMs from ~\cite{ho2020denoising,song2020denoising,song2020score}. Then we introduce the semi-discrete optimal transport (SDOT) used later in this paper.

\subsection{Generalized Diffusion Probabilistic Model}
Given a data distribution $\boldsymbol{x}_0 \sim q(\boldsymbol{x}_0)$, DPMs define a diffusion process $q(\boldsymbol{x}_t | \boldsymbol{x}_{t-1})$ which produces a diffusion trajectory $\{\boldsymbol{x}_{t}\}_{t=1}^{T}$ by adding gaussian noise and a sampling process $p(\boldsymbol{x}_{t-1} | \boldsymbol{x}_{t})$ which reverse the diffusion process to reconstruct the original data. As song et al. remark in \cite{song2020score}, generalized DPMs can be expressed as solutions of stochastic differential equations (SDEs) of the form: 
\begin{equation}
    d\boldsymbol{x} = b(\boldsymbol{x},t)dt + \sigma(t)d\boldsymbol{w}  \label{eq:forwardSDE}
\end{equation}

where $\boldsymbol{w}$ is the standard Winener process (a.k.a., Brownian motion), $b(.,t): \mathbb{R}^d  \rightarrow \mathbb{R}^d $ is vector-valued funciton called the \textit{drift} coefficient of $\boldsymbol{x}(t)$, and $\sigma(\cdot): \mathbb{R} \rightarrow \mathbb{R}$ is a scalar function known as the \textit{diffusion} coefficient of $\boldsymbol{x}(t)$. Eq.~\eqref{eq:forwardSDE} is the limit of the following discrete form (Eq.~\ref{eq:disSDE}) in $\Delta t \rightarrow 0$, which is also known as \textbf{forward SDE}.  
\begin{equation}
    \boldsymbol{x}_{t+\Delta t} = \boldsymbol{x}_t + b(\boldsymbol{x},t)\Delta t + \sigma(t)\sqrt{\Delta t} \boldsymbol{z}, ~~\boldsymbol{z} \sim \mathcal{N}(\mathbf{0}, \mathbf{I})   \label{eq:disSDE}
\end{equation}

From a probability point of view, Eq.~\eqref{eq:disSDE} is reformulated in the following conditional probability:
\begin{equation}
    q(\boldsymbol{x}_{t+\Delta t}|\boldsymbol{x}_t) \sim \mathcal{N} \left( \boldsymbol{x}_t+b(\boldsymbol{x},t)\Delta t, \sigma^{2}(t) \Delta t \mathbf{I}   \right).
    \label{eq:disSDEprob_f}   
\end{equation}

With a sufficiently long diffusion trajectory $\{\boldsymbol{x}_{t}\}_{t=0}^{T}$ and a well-behaved schedule of $\{(b(\cdot,t),\sigma(t))\}_{t=0}^{T}$, the last latent $\boldsymbol{x}_T$ is nearly a Gaussian distribution. Starting from $\boldsymbol{x}_T \sim \mathcal{N}(\mathbf{0,I})$, the exact reverse diffusion distribution $q(\boldsymbol{x}_{t-1}|\boldsymbol{x}_t)$ is indispensable for the sampling process which gradually reverses each step of the noise corruption latents $\boldsymbol{x}_{t-1}$ from $\boldsymbol{x}_{t}$. However, since $q(\boldsymbol{x}_{t-1}|\boldsymbol{x}_t)$ depends on the entire data distribution, DPMs approximate it using a neural network parameterized by $\theta$ as follows:
\begin{equation}
\label{eq:nn}
p_{\theta}(\boldsymbol{x}_{t-1}|\boldsymbol{x}_t) \coloneqq \mathcal{N}(\boldsymbol{x}_{t-1}; \mu_{\theta}(\boldsymbol{x}_t, t), \Sigma_{\theta}(\boldsymbol{x}_t, t))
\end{equation}

Using Bayes rule, the posterior satisfies $p(\boldsymbol{x}_{t-\Delta t}|\boldsymbol{x}_{t}) \sim$ $\mathcal{N} (\boldsymbol{x}_{t} - 
[b(\boldsymbol{x}_{t},t)- 
    \sigma^{2}(t)\nabla_{\boldsymbol{x}_{t}} \textrm{log} q(\boldsymbol{x}_{t})]\Delta t, \sigma^{2}(t) \Delta t \mathbf{I})$. In $\Delta t \rightarrow 0$ , it converges to the following \textbf{inverse SDE}:
\begin{equation} 
d\boldsymbol{x} = [b(\boldsymbol{x},t)-\sigma^{2}(t)\nabla_{\boldsymbol{x}_t} \textrm{log}q(\boldsymbol{x}_t)]dt + \sigma(t)d\boldsymbol{w}. \label{eq:inverseSDE}   
\end{equation}

The estimation of $\nabla_{\boldsymbol{x}_t} \textrm{log}q(\boldsymbol{x}_t)$ is achieved by $\mathbf{s}_{\theta}$. The optimization of $\theta$ can be achieved by minimizing the variational lower bound (Eq.~\ref{eq:loss}) on negative log-likelihood.
\begin{equation}
\begin{aligned}
    L_{vlb} =& -\log p_{\theta}(\boldsymbol{x}_0 | \boldsymbol{x}_1)+  D_{KL}(q(\boldsymbol{x}_T | \boldsymbol{x}_0))||p(\boldsymbol{x}_T))  \\
    & +\sum_{t>1} D_{KL}(q(\boldsymbol{x}_{t-1}|\boldsymbol{x}_t,\boldsymbol{x}_0)|| p_{\theta}(\boldsymbol{x}_{t-1}|\boldsymbol{x}_t))\label{eq:loss} 
\end{aligned}
\end{equation}

After the model is trained well, $\mathbf{s}_\theta$ is a function approximator intended to predict $\boldsymbol{z} \sim \mathcal{N}(\mathbf{0,I})
$ from $\boldsymbol{x}_t$. To sample $\boldsymbol{x}_{t-1}$ from the posterior distribution defined in Eq.~\eqref{eq:nn} is equivalent to inverse diffusion through Eq.~\eqref{eq:inverseDSDE}.
\begin{equation} 
\boldsymbol{x}_{t-\Delta t} = \boldsymbol{x}_{t}-[b(\boldsymbol{x},t)-\sigma^{2}(t)\mathbf{s}_\theta(\boldsymbol{x}_t,t)]\Delta t + \sigma(t)\boldsymbol{z}. \label{eq:inverseDSDE}   
\end{equation}

\subsection{Semi-discrete Optimal Transport} 
Suppose the source measure $\mu$ defined on a convex domain $\Omega \subset \mathbb{R}^{d}$, the target domain is a discrete set $\boldsymbol{Y}=\left\{\boldsymbol{y}_{i}\right\}_{i\in\mathcal{I}}, \boldsymbol{y}_{i} \in \mathbb{R}^{d}$. The target measure is a Dirac measure $\nu=\sum_{i\in\mathcal{I}} \nu_{i} \delta\left(\boldsymbol{y}-\boldsymbol{y}_{i}\right)$ and the source measure is equal to total mass as $\mu(\Omega)=\sum_{i\in\mathcal{I}} \nu_{i} .$ Under a semi-discrete transport map $g: \Omega \rightarrow \boldsymbol{Y}$, a cell decomposition is induced $\Omega=\bigcup_{i\in\mathcal{I}} W_{i}$, such that every $\boldsymbol{x}$ in each cell $W_{i}$ is mapped to the target $\boldsymbol{y}_{i}, g: \boldsymbol{x} \in W_{i} \mapsto \boldsymbol{y}_{i}$. The map $g$ is measure preserving, denoted as $g_{\#} \mu=\nu$, if the $\mu$-volume of each cell $W_{i}$ equals to the $\nu$-measure of the image $g\left(W_{i}\right)=\boldsymbol{y}_{i}, \mu\left(W_{i}\right)=\nu_{i}$. The cost function is given by $c: \Omega \times \boldsymbol{Y} \rightarrow \mathbb{R}$, where $c(\boldsymbol{x}, \boldsymbol{y})$ represents the cost for transporting a unit mass from $\boldsymbol{x}$ to $\boldsymbol{y}$. The total cost of transport map $g(x)$ is given by
\begin{equation}
\int_{\Omega} c(\boldsymbol{x}, g(\boldsymbol{x})) d \mu(\boldsymbol{x})=\sum_{i\in\mathcal{I}} \int_{W_{i}} c\left(\boldsymbol{x}, \boldsymbol{y}_{i}\right) d \mu(\boldsymbol{x}).
\label{eq:cost}
\end{equation}
The SDOT map $g^{\ast}$ is a measure-preserving map that minimizes the total cost in Eq.~\eqref{eq:cost},
\begin{equation}
g^{\ast}:=\arg \min _{g_{\#} \mu=\nu} \int_{\Omega} c(\boldsymbol{x}, g(\boldsymbol{x})) d \mu(\boldsymbol{x}).\label{eq:SDOT}
\end{equation}

Based on \textbf{Theorem 1.1} of \textbf{supplementary material}, when the cost function $c(\boldsymbol{x}, \boldsymbol{y})=1 / 2\|\boldsymbol{x}-\boldsymbol{y}\|^{2}$, we have
$g^{*}(\boldsymbol{x})=\nabla \boldsymbol{u}(\boldsymbol{x}).$
This explains that the SDOT map is the gradient map of Brenier's potential $\boldsymbol{u}$. As~\cite{lei2020geometric,an2019ae} remark, $\boldsymbol{u}$ is the upper envelope of a collection of hyperplanes $\pi_{\boldsymbol{h},i}(\boldsymbol{x}) =  \boldsymbol{x}^T\boldsymbol{y}_i+h_i$ and can be parametrized uniquely up to an additive constant by a height vector $\boldsymbol{h} = (h_1, h_2, ..., h_{\left|\mathcal{I} \right|})^T$.  In such a case,
$\boldsymbol{u}_{\boldsymbol{h}}$ parameterized by $\boldsymbol{h}$ can be stated as follows,
\begin{equation}
\boldsymbol{u}_{\boldsymbol{h}}(\boldsymbol{x}) = \max_{i\in\mathcal{I}}\{\pi_{\boldsymbol{h},i}(\boldsymbol{x})\}, \boldsymbol{u}_{\boldsymbol{h}}: \Omega \rightarrow \mathbb{R}^n,\label{eq:uh}
\end{equation}

Given the target measure $\nu$, there exists Brenier's potential $\boldsymbol{u}_{\boldsymbol{h}}$ in Eq.~\eqref{eq:uh} whose projected volume of each support plane is equal to the given target measure $\nu_i$. To receive $\boldsymbol{u}_{\boldsymbol{h}}$, we only need to optimal $\boldsymbol{h}$ by minimizing the following convex energy function:
\begin{equation}
E(\boldsymbol{h})=\int_{\boldsymbol{0}}^{\boldsymbol{h}} \sum_{i\in\mathcal{I}} w_{i}(\eta) d \eta_{i}-\sum_{i\in\mathcal{I}} h_{i} \nu_{i},
\label{eq:EH}
\end{equation}
where $\omega_{i}(\eta)$ is the $\mu$-volume of $W_{i}(\eta)$.

\section{Diffusion Probabilistic Model Based on Optimal Transport} \label{sec:3}
\begin{figure*} 
\vskip -0.1in
\centering
\begin{center}
\centerline{\includegraphics[width=0.9\textwidth]{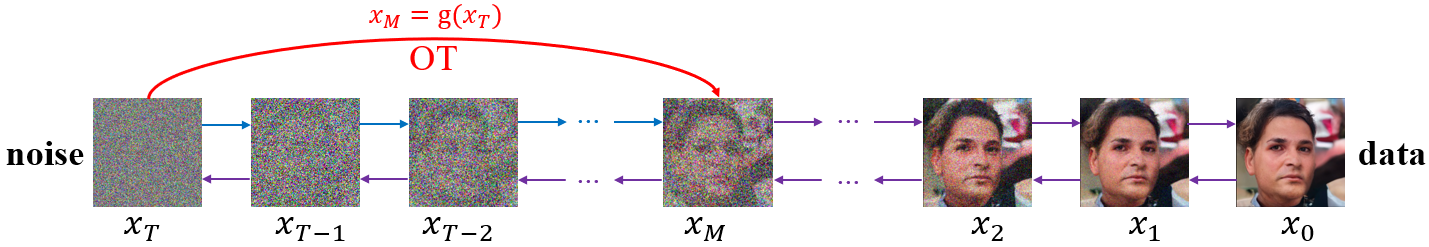}}
\caption{The framework of the proposed \textbf{DPM-OT}. The red curve indicates the \textbf{Optimal Trajectory}, which is induced by the SDOT map $g\left( \cdot\right)$ between $\boldsymbol{x}_T$ and $\boldsymbol{x}_M$. Correspondingly, the blue line indicates the first $T$-$M$ steps inverse diffusion of the vanilla DPM.}
\label{framework}
\end{center}
\vskip -0.3in
\end{figure*}

Leveraging the generative capabilities of DPMs and the distribution-aligned nature of OT, we propose a fast DPM whose framework is shown in Fig.~\ref{framework}. We give the definition of our \textbf{DPM-OT} sampler in Section~\ref{sec:OptimalTrajectory} and error analysis in Section~\ref{sec:analysis}. Section~\ref{sec:sampling} describes the proposed fast DPM from the perspective of the algorithm.

\subsection{Optimal Trajectory and Sampler}\label{sec:OptimalTrajectory}
Effective trajectory shortening method has been shown significant for sampling acceleration~\cite{watson2021learning,bao2022analytic,bao2022estimating} and high-fidelity generation. To this end, Eric and Salimans et al.~\cite{luhman2021knowledge,salimans2022progressive} recursively distill the 2-step trajectory in the teacher network into a single-step using the knowledge distillation technique. Bao et al.~\cite{bao2022analytic} and Watson et al. ~\cite{watson2021learning} use dynamic programming to estimate optimal trajectory, which can quickly generate high-quality images.
Bao et al.~\cite{bao2022estimating} propose to estimate the optimal covariance and its correction given imperfect means by learning these conditional expectations at each time step. Inspired by that prior wisdom, we propose a new SDOT based diffusion model \textbf{DPM-OT}, which builds a direct expressway represented by optimal trajectory. Thus, the single-step optimal trajectory replaces the multi-step trajectory in the vanilla DPM. The definition of optimal trajectory is given below.
\begin{definition}
\label{def:optimaltrajectory}
\textbf{(Optimal Trajectory)}. Given a $M$ steps trajectory $\{\boldsymbol{x}_{t-i}\}_{i=0}^{M}$ at time $t$, $M\leq t$, $\boldsymbol{x}_t \in \boldsymbol{X}_t$, the optimal trajectory from $\boldsymbol{x}_t$ to $\boldsymbol{x}_{t-M}$ is a single-step trajectory that is obtained by minimizing the following transport cost:
\begin{equation}
g^{\ast}:=\arg \min _{g_{\#} \mu=\nu} \int_{\boldsymbol{X}_t} c(\boldsymbol{x}_{t-M}, g(\boldsymbol{x}_t)) d \mu(\boldsymbol{x}_t).
\label{eq:optimaltrajectorycost}
\end{equation}
\end{definition}

According to Brenier’s theorem, the OT map $g^{\ast}(\cdot)$ is the gradient of the convex Brenier potential $\boldsymbol{u}_{\boldsymbol{h}}$ which satisfies the Monge-Ampére equation. The existence and the uniqueness of the solution to the Monge-Ampére equation have been proved by the Fields medalist Figalli in Chapter 2 of \cite{figalli2017monge}, where he used Alexandrov’s approach and claimed:
\romannumeral1). The sequence of Dirac distributions  $\left\{\nu_n\right\}$ weakly converges to data distribution $\nu$;
\romannumeral2). For each Dirac measure $\nu_n$, there exists an Alexandrov’s solution $\boldsymbol{u}_n$ (which is exactly the discrete Brenier potential in this paper);
\romannumeral3). The weak solution $\boldsymbol{u}_n$ converges to the real solution $\boldsymbol{u}_{\boldsymbol{h}}$ which is $C^1$ almost everywhere, except at the singular set. Thus, the OT map $g^{\ast}(\cdot)$ is continuous internally and discontinuous at the boundary singular set.

\begin{definition}
\label{def:DPM-OT}
\textbf{(DPM-OT Sampler)}. Given an initial latent code $\boldsymbol{x}_T \sim \mathcal{N}(\boldsymbol{0,I})$ at time $T$, the DPM-OT sampler needs to go through a $M+1$ steps trajectory $\boldsymbol{x}_T \cup \{\boldsymbol{x}_i\}_{i=M}^{0}$, where $M<T$. Let $g(\cdot)$ denote the OT map between $\boldsymbol{x}_T$ and $\boldsymbol{x}_{M}$, $f_t(\cdot,\cdot)$ denote the parameterized reverse diffusion process which can be any off-the-shelf DPM model. The DPM-OT sampler is defined as follows:
\begin{equation}
\label{eq:DPM-OT}
  \begin{split}
    \boldsymbol{x}_M &= g(\boldsymbol{x}_T),\\
    \boldsymbol{x}_{t-1}&=f_t(\boldsymbol{x}_t,\boldsymbol{z}), \quad t=M,...,1.
  \end{split}
\end{equation}
\end{definition}

The \textbf{DPM-OT Sampler} first transmits the white noise $\boldsymbol{x}_T$ to the manifold represented by the latent variable $\boldsymbol{x}_M$ through the OT map $g(\cdot)$, thus providing a near-perfect initial value for the subsequent inverse diffusion process. From \textbf{Definition}~\ref{def:optimaltrajectory}, $g(\cdot)$ is discontinuous at the singular set. Therefore, the manifold $\boldsymbol{x}_M$ can maintain the same attributes as the original data manifold, that is, discontinuity at the boundary singular point, thus avoiding mode mixture. Further, the subsequent $M$-step inverse diffusion gradually pushes latent variable $\boldsymbol{x}_M$ onto the target data manifold, and finally achieves high-quality data generation. The variable $\boldsymbol{z}$ in $f_t(\cdot,\cdot)$ is $\boldsymbol{0}$ when $t=1$ or $\boldsymbol{z} \sim \mathcal{N}(\mathbf{0,I})$ when $t>1$. $M$-step inverse diffusion process improves the generation ability of \textbf{DPM-OT Sampler}, which essentially completes the continuity of SDOT map.

\begin{algorithm}[tb]
   \caption{SDOT Map}
   \label{alg:OTMap}
\begin{algorithmic}
   \REQUIRE  Target dataset $\boldsymbol{Y}=\left\{\boldsymbol{y}_i\right\}_{i\in\mathcal{I}}$ with empirical distribution $\nu=\frac{1}{|\mathcal{I}|}\sum_{i\in\mathcal{I}} \nu_{i} \delta\left(\boldsymbol{y}-\boldsymbol{y}_{i}\right)$, number of Monte Carlo samples $N$, learning rate $lr$, threshold $\tau$, positive integer $s$, reverse diffusion steps $M$ and a well-behaved schedule of $\{(b_t,\sigma_t)\}_{t=0}^{T}$.
   \ENSURE OT map $g(\cdot)$.
   \STATE Initialize $\boldsymbol{h}=(h_1,h_2,\cdots,h_{|\mathcal{I}|})\leftarrow(0,0,\cdots,0)$
   \STATE Diffuse $\boldsymbol{y}_i$ forwardly by $M$ steps according to Eq.~\eqref{eq:disSDE}, obtain the latents set $\boldsymbol{X}_M=\{\boldsymbol{x}_{M}^i\}_{i\in\mathcal{I}}$
   \REPEAT
    \STATE Sample $N$ white noise samples $\left\{\boldsymbol{x}_T\sim \mathcal{N}(0,I)\right\}_{j=1}^N$
    \STATE Calculate $\nabla \boldsymbol{h}=(\hat{w}_i(\boldsymbol{h})-\nu_i )^T$.
    \STATE $\nabla \boldsymbol{h}=\nabla \boldsymbol{h}-mean(\nabla \boldsymbol{h})$.
    \STATE Update $\boldsymbol{h}$ by Adam algorithm with $\beta_1=0.9$,$\beta_2=0.5$.
    \STATE Calculate $ E(\boldsymbol{h})$ by Eq.~\eqref{eq:EH}
    \IF{$ E(\boldsymbol{h})$ has not decreased for $s$ steps}
    \STATE $N\leftarrow 2\times N$;\quad $lr\leftarrow 0.8\times lr$
    \ENDIF
   \UNTIL{$E(\boldsymbol{h})<\tau$}
   \STATE OT map $g(\cdot)\leftarrow \nabla(\max_{i}\{\langle \boldsymbol{x}_T,\boldsymbol{x}_M^i\rangle_F  + h_i\})$.
\end{algorithmic}
\end{algorithm}

\subsection{Sampling Algorithm} \label{sec:sampling}
The overall framework of our \textbf{DPM-OT} is shown in Fig.~\ref{framework}, which includes an optimal trajectory from $\boldsymbol{x}_T$ to $\boldsymbol{x}_{M}$ and an $M$-step inverse diffusion process gradually pushing latent variable $\boldsymbol{x}_M$ onto the target data manifold. We summarize our framework in \textbf{Algorithm~\ref{alg:OTMap}} and \textbf{Algorithm~\ref{alg:sampling}}. Specifically, to sample high-quality images, we need to compute the SDOT map $g:\boldsymbol{x}_T\rightarrow \boldsymbol{x}_M$ by \textbf{Algorithm~\ref{alg:OTMap}}, and then use \textbf{Algorithm~\ref{alg:sampling}} to generate images. 

Given the target dataset $\boldsymbol{Y}=\left\{\boldsymbol{y}_i\right\}_{i\in\mathcal{I}}$, our goal is to efficiently generate high-quality images distributed on the manifold represented by $\boldsymbol{Y}$. Intuitively, each point $\boldsymbol{y}_i$ in $\boldsymbol{Y}$ is distributed on the target manifold discretely, so the manifold induced by $\boldsymbol{Y}$ is often discontinuous, especially on the class boundary. Unfortunately, in the generation task, what the community cares about is the coverage and alignment of the entire target data manifold, while the properties of the manifold itself such as the discrete nature of the boundaries are often ignored. As a result, existing DPMs try to approximate a continuous diffusion process with a deep neural network, which leads to mode mixture in their output images.

To accelerate the sampling process and avoid the mode mixture, we first utilize a well-behaved schedule $\{(b_t,\sigma_t)\}_{t=0}^{T}$ to diffuse the original data $\boldsymbol{y}$ into the latent variable $\boldsymbol{x}_{M}$ and then calculate the OT map between $\boldsymbol{x}_T$ and $\boldsymbol{x}_M$, which induces the \textbf{Optimal Trajectory} from $\boldsymbol{x}_T$ to $\boldsymbol{x}_M$.  Considering that both $\boldsymbol{x}_T$ and $\boldsymbol{x}_M$ are matrices, we set the Brenier's potential $\boldsymbol{u}_{\boldsymbol{h}}=\max_{i\in\mathcal{I}}\{ \langle\boldsymbol{x}_T,\boldsymbol{x}_M^i\rangle_F + h_i\}$, where $\langle \cdot,\cdot\rangle_F$ denotes Frobenius inner product. In \textbf{Algorithm~\ref{alg:OTMap}}, we use the Monte Carlo method to solve the SDOT map. For better convergence, we double the number of samples $N$ and multiply the learning rate $lr$ by 0.8 when the energy function $E(\boldsymbol{h})$ has not decreased for $s$ steps.

After obtaining the SDOT map $g(\cdot)$, we sample with the help of an off-shelf pre-trained model and summarize the process in \textbf{Algorithm~\ref{alg:sampling}}. With the optimal trajectory, our \textbf{DPM-OT Sampler} can perform sampling quickly and with few mode mixture. Given different schedule $\{(b_t,\sigma_t)\}_{t=0}^{T}$, our framework \textbf{DPM-OT} will be instantiated into different fast DPM models. Specifically, if we use the Langevin dynamic in NCSNv2 \cite{song2020improved} to instantiate the diffusion process, we can obtain the following sampling process:
\begin{equation}\label{eq:Langevin-sap}
    \boldsymbol{x}_{t-1} = \boldsymbol{x}_{t}+\sigma^{2}(t)\mathbf{s}_\theta(\boldsymbol{x}_t,t) + \sigma(t)\boldsymbol{z}, \  t=M,...,1.
\end{equation}


\begin{algorithm}[tb]
   \caption{\textbf{DPM-OT} Sampling}
   \label{alg:sampling}
\begin{algorithmic}
   \REQUIRE Reverse diffusion steps $M$, OT map $g(\cdot)$, a well-trained $\mathbf{s}_\theta$ and a well-behaved schedule of $\{(b_t,\sigma_t)\}_{t=0}^{T}$.
   \ENSURE Generated image $\boldsymbol{x}_0$.
   \STATE Sample $\boldsymbol{x}_T\sim \mathcal{N}(0,I)$
   \STATE $\boldsymbol{x}_M = g(\boldsymbol{x}_T)$
   \FOR{$t=M$ {\bfseries to} $1$}
   \STATE $\boldsymbol{z}\sim \mathcal{N}(0,I)$ if $t > 1$, else $\boldsymbol{z}=0$ 
   \STATE $\boldsymbol{x}_{t-1} = \boldsymbol{x}_{t}-[b(\boldsymbol{x},t)-\sigma^{2}(t)\mathbf{s}_\theta(\boldsymbol{x}_t,t)] + \sigma(t)\boldsymbol{z}$
   \ENDFOR
   \STATE {\bf return} $\boldsymbol{x}_0$
\end{algorithmic}
\end{algorithm}

\subsection{Error Analysis} \label{sec:analysis}
 In this section, we analyze the error bound of \textbf{DPM-OT}. First, we prove that the single-step error is controllable. Then, we give the upper bound of the error between the generated data distribution and the target data distribution.
\begin{theorem}\label{theorem:oneerro}
Let $\Tilde{\boldsymbol{x}}_t$ and $\boldsymbol{x}_t$ be the samples of step $t$ obtained by \textbf{DPM-OT} and forward diffusion respectively, and $t\leqslant M$, $\boldsymbol{\zeta}_M$ be the error at step $M$ induced by optimal trajectory, then there is a constant $C_t>0$ satisfies 
\begin{equation}
    \left\| \boldsymbol{\Tilde{x}}_t-\boldsymbol{x}_t \right\|\leqslant C_t \left\| \boldsymbol{\zeta}_M \right\|.
\end{equation} 
\end{theorem}

Since the weak solution $\boldsymbol{u}_n$ converge to the real solution $\boldsymbol{u}_{\boldsymbol{h}}$ \cite{figalli2017monge,an2019ae,lei2020geometric} and the OT map $g(\cdot)$ is obtained by the Monte Carlo method, its error $\boldsymbol{\zeta}_M$ is $O(N^{-\frac{1}{2}})$ according to \textbf{Theorem 2.1} of ~\cite{caflisch1998monte}. So we can find a small enough error $\boldsymbol{\zeta}_M$ to make $\left\|\boldsymbol{\Tilde{x}_t}-\boldsymbol{x_t} \right\|\leqslant \delta$ hold for any given error bound $\delta>0$, which indicates the error of \textbf{DPM-OT} is  controllable. Building upon \textbf{Theorem~\ref{theorem:oneerro}}, we obtain \textbf{Theorem~\ref{theorem:errobound1}} and give their proofs in \textbf{supplementary material}.
\begin{theorem}\label{theorem:errobound1}
    Let $L_{dpm\_ot}$ be the error between the data distribution generated by \textbf{DPM-OT} and the target data distribution which is defined in Eq.~\eqref{eq:erroDPMOT}, $L_{vlb}$ is the variational lower bound on negative log-likelihood between data distribution generated by vanilla DPM and the target data distribution which is defined in Eq.~\eqref{eq:loss}. We have $L_{dpm\_ot} \leqslant L_{vlb}$. Hence, $L_{vlb}$ is the upper bound of $L_{dpm\_ot}$.
    \begin{equation}\label{eq:erroDPMOT}
    \begin{split}
    L_{dpm\_ot}&=L_0 + L_1 + ... + L_M+L_{T}\\
    &=-\log \tilde{p}_{\theta}(\boldsymbol{x}_0 | \boldsymbol{x}_1)+D_{KL}(q(\boldsymbol{x}_T | \boldsymbol{x}_0)),p(\boldsymbol{x}_T))\\
    &+\sum_{t=1}^{M-1} D_{KL}(q(\boldsymbol{x}_t|\boldsymbol{x}_{t+1},\boldsymbol{x}_0)|| \tilde{p}_{\theta}(\boldsymbol{x}_t|\boldsymbol{x}_{t+1}))\\
    &+D_{KL}(q(\boldsymbol{x}_M|\boldsymbol{x}_T,\boldsymbol{x}_0)|| \tilde{p}_{\theta}(\boldsymbol{x}_M|\boldsymbol{x}_T)),
    \end{split}
    \end{equation}
     \vspace{-2mm}
\end{theorem}

\noindent where $\tilde{p}_{\theta}(\boldsymbol{x}_M|\boldsymbol{x}_T)=p_{\theta}(\boldsymbol{x}_M-\boldsymbol{\zeta}_M|\boldsymbol{x}_T)$, $\tilde{p}_{\theta}(\boldsymbol{x}_t|\boldsymbol{x}_{t+1})=p_{\theta}(\boldsymbol{x}_t-\boldsymbol{\zeta}_t|\boldsymbol{x}_{t+1})$ and $\boldsymbol{\zeta}_t=\boldsymbol{\Tilde{x}}_t-\boldsymbol{x}_t$ for $t=0$ to $t=M-1$. \textbf{Theorem~\ref{theorem:errobound1}} shows that our \textbf{DPM-OT} sampler can fit the target data distribution no worse than vanilla DPMs, which theoretically guarantees the robustness of the algorithm.

\begin{table}[t]
\caption{Sample quality measured by FID $\downarrow$ on CIFAR-10, CelebA,  and FFHQ, the number of function evaluations(NFE) is the number of times the neural network called.}
\label{tab:FID}
\begin{center}
\scalebox{0.91}{
\begin{tabular}{l|lcc}
\toprule
    Dataset &Models  &  NFE  & FID $\downarrow$ \\  
    \midrule 
    \multirow{10}{*}{\textbf{CIFAR-10}}
   & DDIM~\cite{song2020denoising} & 100 &  4.60  \\
   & NCSNv2~\cite{song2020improved}  & 1160 &  10.23   \\
   & DPM-Solver~\cite{lu2022dpm}  & 20 &  3.72  \\
   & EDM~\cite{karras2022elucidating}  & 27 & 3.73  \\
   & UniPC~\cite{zhao2023unipc}   & 10 & 3.87  \\
    \cline{2-4}
       &  & 5 & 3.78  \\
      &   & 10 &  3.61  \\
   & \textbf{DPM-OT}  & 20 &  3.33   \\
     &   & 30 &  3.12  \\
     &   & 50& \textbf{2.92}  \\
    \midrule
    \midrule
    \multirow{10}{*}{\textbf{CelebA $\mathbf{64\times 64}$}}
   &  DDIM~\cite{song2020denoising} & 100 &  6.53 \\
   &  NCSNv2~\cite{song2020improved} & 2500 &  10.23   \\
    & DPM-Solver~\cite{lu2022dpm} & 20 & 3.13  \\
    & Analytic-DDIM~\cite{bao2022analytic} & 10  & 15.62   \\
   &  Analytic-DDIM~\cite{bao2022analytic} & 50  & 6.13   \\
   \cline{2-4}
   &      & 5 & 3.30  \\
   &      & 10 &3.21  \\
   &  \textbf{DPM-OT}  & 20 & 3.12  \\
   &      & 30 & 3.01  \\
   &      & 50& \textbf{2.85}   \\
    \midrule
    \midrule
    \multirow{8}{*}{\textbf{FFHQ $\mathbf{256\times 256}$}}
   &  DPM-Solver~\cite{lu2022dpm} & 10 & 7.39  \\
   &  UniPC~\cite{zhao2023unipc} & 10 & 6.99  \\
   &  NCSNv2~\cite{song2020improved} & 6933 & 12.73   \\
   \cline{2-4}
   &    & 5 & 4.69   \\
    &   & 10 & 4.46    \\
   &  \textbf{DPM-OT}  & 20 & 4.32    \\
   &   & 30 & 4.26    \\
   &     & 50 & \textbf{4.11}    \\
    \midrule
 
\end{tabular}
}
\end{center}
\vskip -0.1in
\end{table}
 

\begin{table}[t]
\caption{Comparison of precision $\uparrow$ and recall $\uparrow$ .}
\label{tab:PR}
\begin{center}
\scalebox{0.85}{
\begin{tabular}{l|lcc}
\toprule
    Dataset & Models  & Precision $\uparrow$ & Recall $\uparrow$\\
    \midrule
    
    
    \multirow{4}{*}{\textbf{CelebA $\mathbf{64\times 64}$}}
    & DDIM~\cite{song2020denoising}    & 0.75 & 0.42\\
    &  NCSNv2~\cite{song2020improved}  & \textbf{0.85} & 0.42\\
     & DPM-Solver~\cite{lu2022dpm}    &  0.71 & 0.46\\
     & \textbf{DPM-OT}      &  \underline{0.79} & \textbf{0.78} \\
    \midrule
   
\end{tabular}
}
\end{center}
\vskip -0.2in
\end{table}

\begin{table*}
\caption{Comparison of Mode mixture in \textbf{CIFAR-10} dataset. NIG: the number of images generated, threshold: the image has a probability greater than a
given threshold on more than two categories, it is identified as mode mixture. }
\label{tab:modeMixture}
\begin{center}
\begin{tabular}{l|lcccccc}
\toprule
  \multirow{2}{*}{NIG}   & \multirow{2}{*}{models} & \multirow{2}{*}{NFE}  & \multicolumn{5}{c}{threshold}  \\
     \cmidrule{4-8}
        &  & & $\lambda=0.1$  & $\lambda=0.11$ & $\lambda=0.13$ & $\lambda=0.16$ & $\lambda=0.2$   \\
    \midrule
    49984 & DDIM &100 & 3834($7.69 \%$) &3626($7.25 \%$) &3306($6.61 \%$) &2859($5.71 \%$)  &2416($4.83 \%$)\\
    50000 & DPM-solver &20 & 3824($7.65 \%$) &3614($7.22 \%$) &3295($6.59 \%$) &2868($5.73 \%$)  &2397($4.79 \%$)    \\
    50000 & EDM &27 & 1579($3.16 \%$) &1512($3.02 \%$) &1379($2.75 \%$) &1203($2.41 \%$)  &980($1.96 \%$)  \\
    50000 & NCSNv2 & 1160 & 8876($17.75\%$) &8217($16.43 \%$) &7463($14.92 \%$) &6527($13.05 \%$)  &5552($11.10 \%$)   \\
    \midrule
      &   &5 & \textbf{699($\mathbf{1.40 \%}$)}&\textbf{743($\mathbf{1.48 \%}$)}&\textbf{678($\mathbf{1.35\%}$)} &\textbf{585($\mathbf{1.17\%}$)} &\textbf{426($\mathbf{0.85 \%}$)}   \\
    &   &10 & \textbf{778($\mathbf{1.56 \%}$)}&\textbf{746($\mathbf{1.49 \%}$)}&\textbf{674($\mathbf{1.34\%}$)} &\textbf{590($\mathbf{1.18\%}$)} &\textbf{492($\mathbf{0.98 \%}$)}   \\
    50000 & \textbf{DPM-OT} &20 & \textbf{829($\mathbf{1.66 \%}$)}&\textbf{835($\mathbf{1.67 \%}$)}&\textbf{746($\mathbf{1.49\%}$)} &\textbf{645($\mathbf{1.29\%}$)} &\textbf{527($\mathbf{1.05 \%}$)} \\
      &   &30 & \textbf{840($\mathbf{1.68 \%}$)}&\textbf{834($\mathbf{1.66 \%}$)}&\textbf{749($\mathbf{1.49\%}$)} &\textbf{645($\mathbf{1.29\%}$)} &\textbf{496($\mathbf{0.99 \%}$)}   \\
      &   &50 & \textbf{984($\mathbf{1.97 \%}$)}&\textbf{951($\mathbf{1.90 \%}$)}&\textbf{872($\mathbf{1.74\%}$)} &\textbf{760($\mathbf{1.52\%}$)} &\textbf{622($\mathbf{1.24 \%}$)}    \\
\bottomrule
\end{tabular}
\end{center}
\vskip -0.1in
\end{table*}
\begin{figure*}[h]
\centering
\subfigure[DDIM]{
\includegraphics[width=0.184\textwidth]{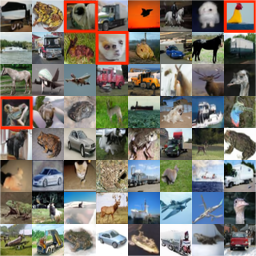}}
\subfigure[NCSNv2]{
\includegraphics[width=0.184\textwidth]{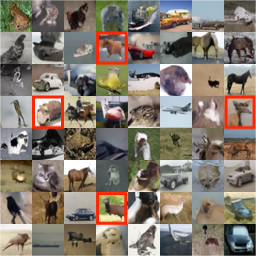}}
\subfigure[EDM]{
\includegraphics[width=0.184\textwidth]{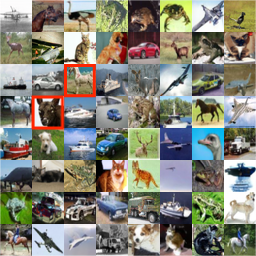}}
\subfigure[DPM-solver]{
\includegraphics[width=0.184\textwidth]{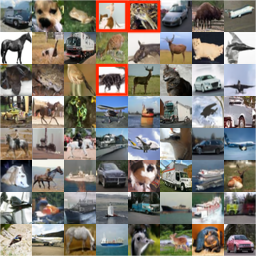}}
\subfigure[DPM-OT]{
\includegraphics[width=0.184\textwidth]{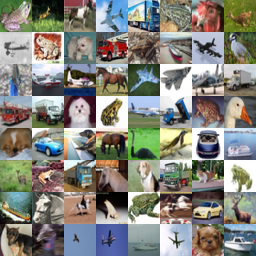}}
\setlength{\abovecaptionskip}{0cm}  
\setlength{\belowcaptionskip}{0cm} 
\caption{The visual comparison on CIFAR-10 dataset.}\label{figure:cifar}
\end{figure*}
\begin{figure*}[h]
\centering
\subfigure[DDIM]{
\includegraphics[width=0.23\textwidth]{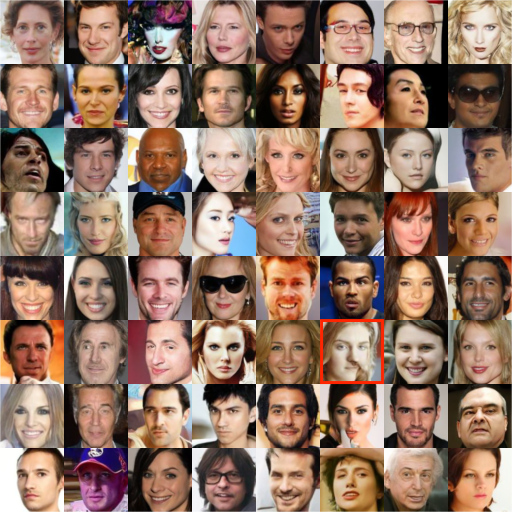}}
\subfigure[NCSNv2]{
\includegraphics[width=0.23\textwidth]{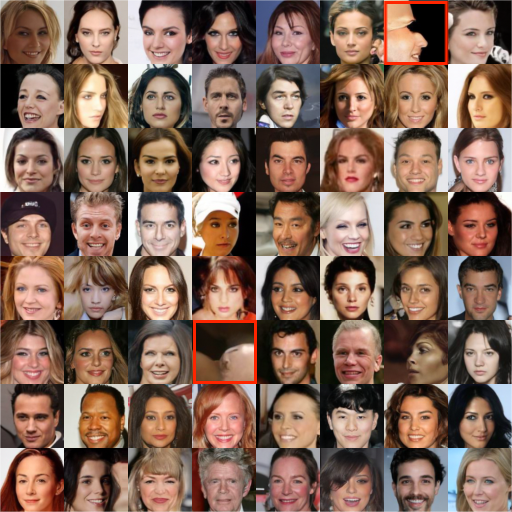}}
\subfigure[DPM-solver]{
\includegraphics[width=0.23\textwidth]{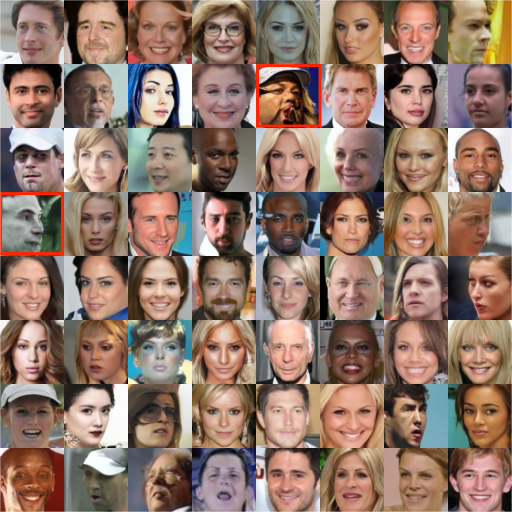}}
\subfigure[DPM-OT]{
\includegraphics[width=0.23\textwidth]{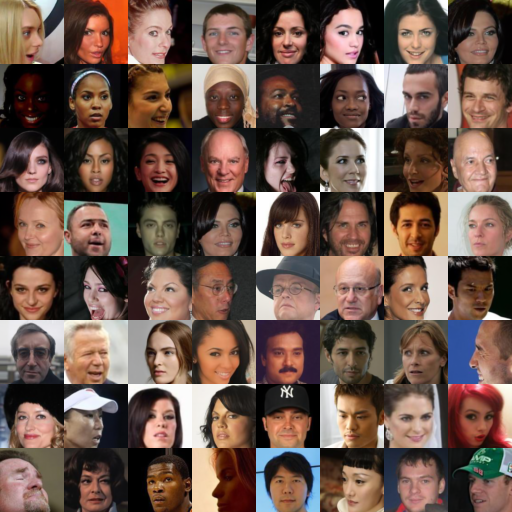}}
\setlength{\abovecaptionskip}{0cm}  
\setlength{\belowcaptionskip}{0cm} 
\caption{The visual comparison on CelebA dataset.}\label{figure:celeba}
\vskip -0.2in
\end{figure*}
\begin{figure}[ht]
\centering
\subfigure[NCSNv2]{
\includegraphics[width=0.39\textwidth]{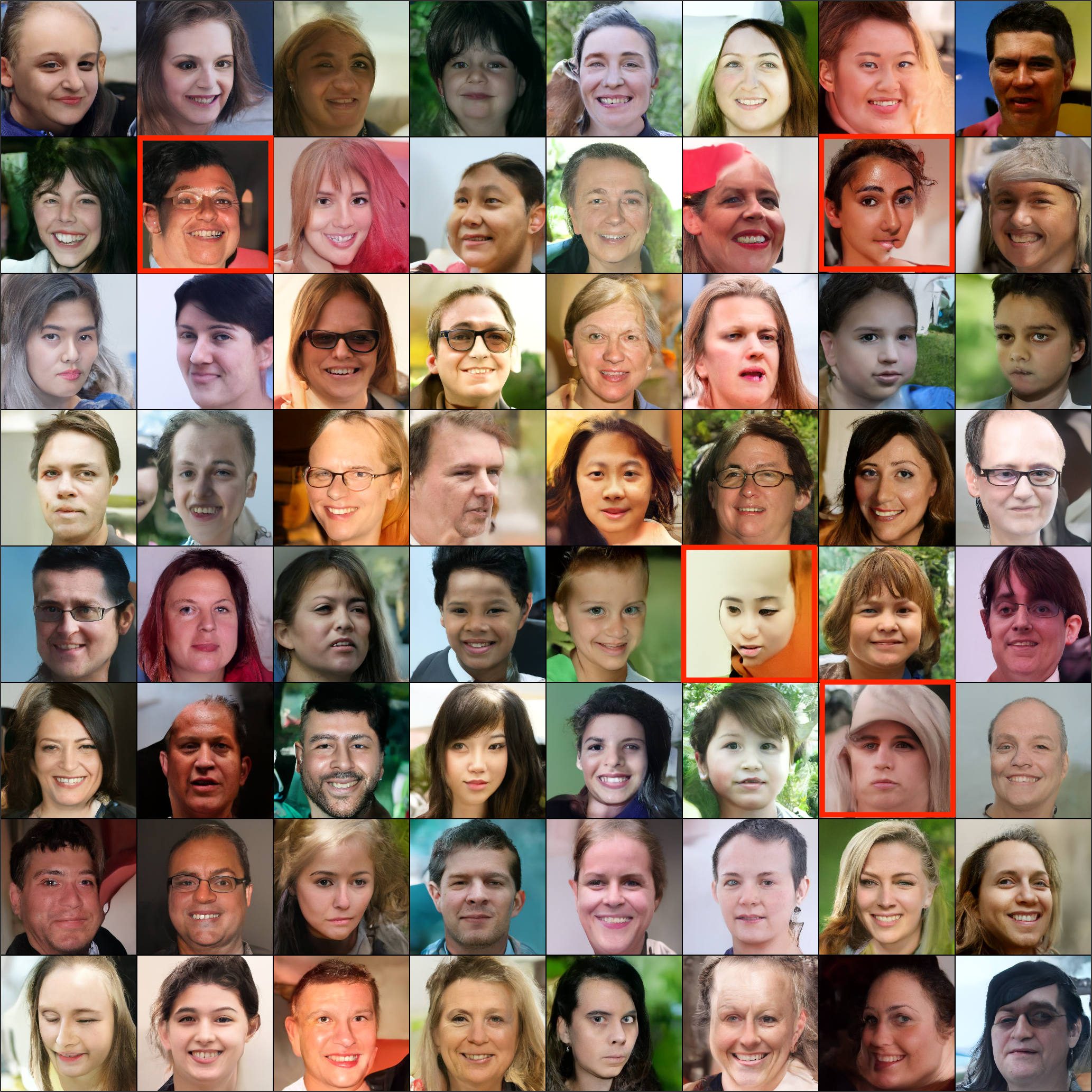}}
\subfigure[DPM-OT]{
\includegraphics[width=0.39\textwidth]{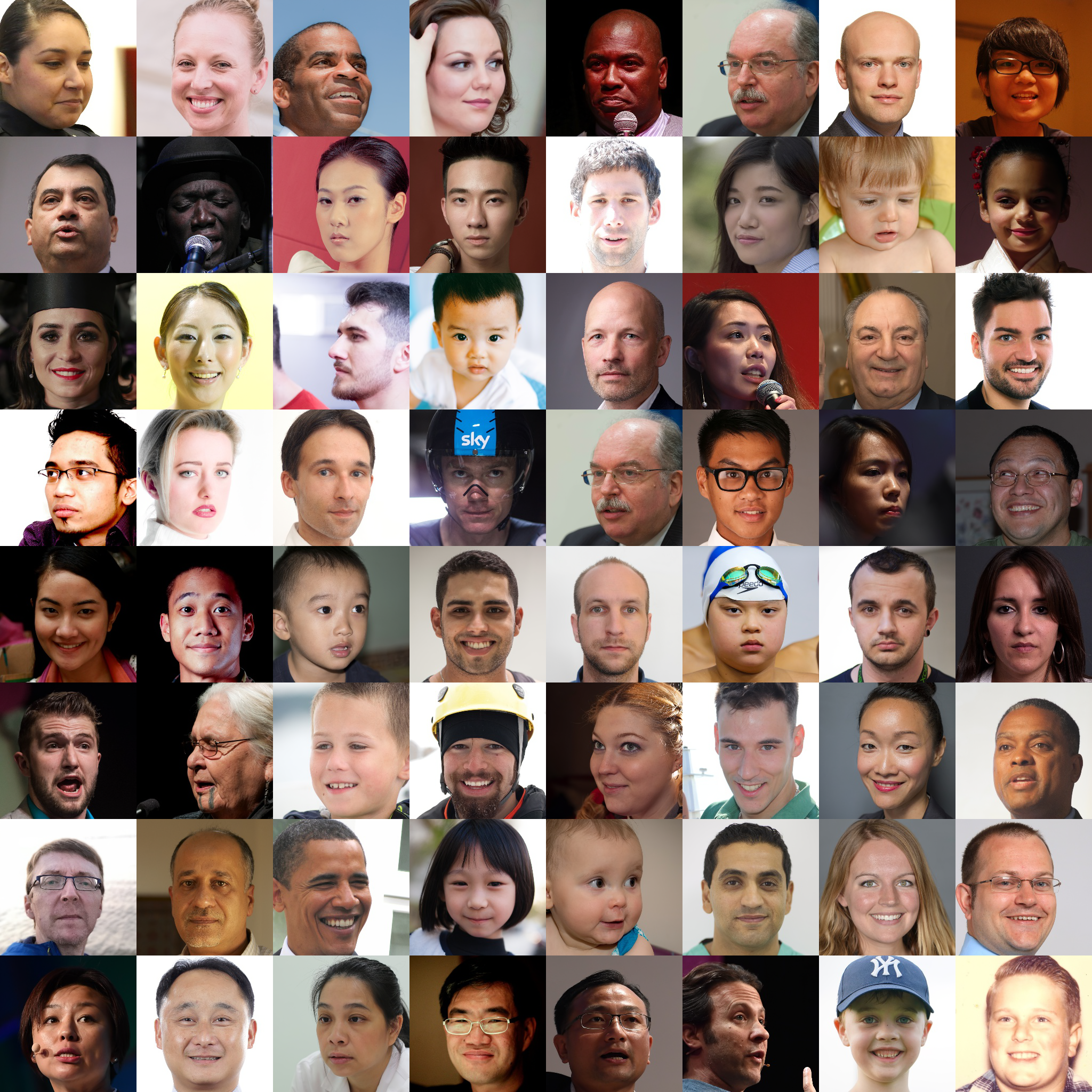}}
\setlength{\abovecaptionskip}{0cm}  
\setlength{\belowcaptionskip}{0cm} 
\caption{The visual comparison on FFHQ dataset.}\label{figure:ffhq}
\vskip -0.2in
\end{figure}
\begin{figure*}[h]
\centering
\includegraphics[width=0.85\textwidth]{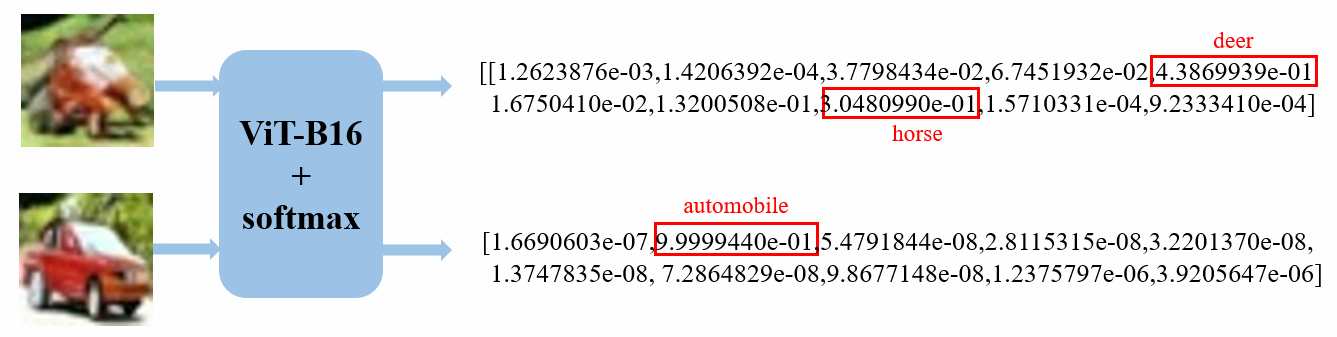}
\setlength{\abovecaptionskip}{0cm}  
\setlength{\belowcaptionskip}{0cm} 
\caption{Examples of verifying mode mixture indicator. The image of the top row and the bottom row are judged as existence and nonexistence mode mixture by the indicator in Eq.~\eqref{eq:indicative}.}\label{figure:mixexample}
\end{figure*}

\begin{figure}[ht]
\centering
\includegraphics[width=0.45\textwidth]{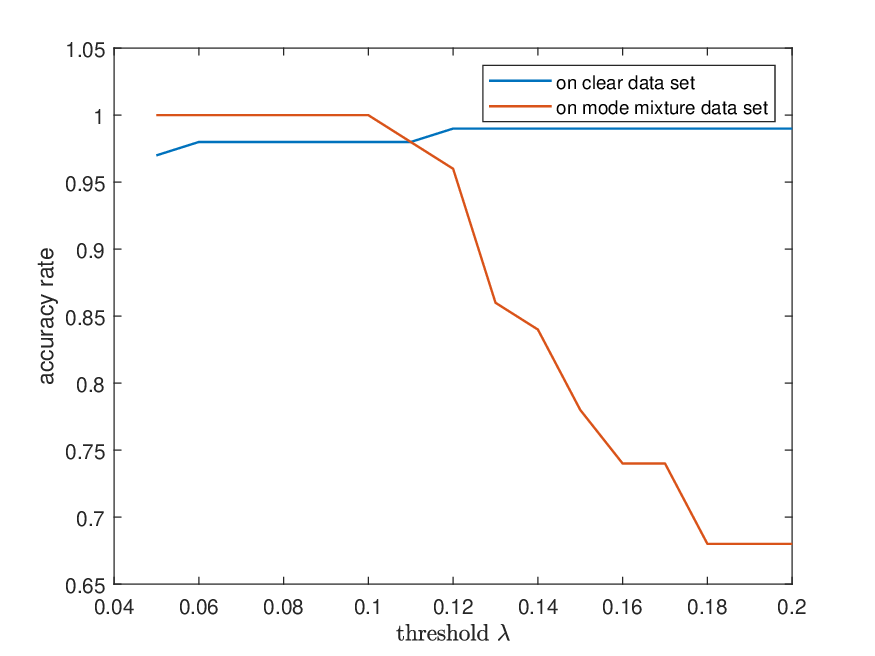}
\setlength{\abovecaptionskip}{0cm}  
\setlength{\belowcaptionskip}{0cm} 
\caption{Randomly chosen mixed and unambiguous images to obtain the mode mixture indicator results under different thresholds.}\label{figure:acc_thred}
\end{figure}

\section{Experiments}
Section \ref{sec:3} analyzes the benefits of our model from the theoretical perspective, and then we will further evaluate the performance of the model experimentally. The experimental results indicate that the \textbf{DPM-OT} has the following pros: {\textbf{1)} it can be well embedded in pre-trained diffusion models to accelerate sampling;  \textbf{2)} the generation efficiency of our model has been drastically enhanced, and high-quality images can be generated with only 5 function evaluations; 
\textbf{3)} mode mixture can be alleviated via the proposed method. 

More specifically, we instantiate our framework \textbf{DPM-OT} with pre-trained models of NCSNv2 \cite{song2020improved} on CIFAR10~\cite{krizhevsky2009learning}, CelebA~\cite{liu2015deep} and FFHQ~\cite{Karras_2019_CVPR} respectively. Compared to the NCSNv2, our method has a great improvement in sampling speed. NCSNv2 needs 1160, 2500, and 6933 function evaluations on the corresponding three datasets to get preferable images, while our model can get high-quality images with only 5 function evaluations. Furthermore, we employ FID score \cite{heusel2017gans} and the improved precision and recall metric \cite{kynkaanniemi2019improved} to assess its performance against the SOTA models on the above three datasets. These SOTA models include DDIM~\cite{song2020denoising}, Analytic-DDIM~\cite{bao2022analytic}, DPM-Solver~\cite{lu2022dpm}, EDM~\cite{karras2022elucidating}, UniPC~\cite{zhao2023unipc}.

To measure the ability of generative models to generate images with obvious categories, that is, not to produce images that do not exist in the real world with multiple categories mixed together,  we design a mode mixture indicator to judge whether the image is mixed. Exactly, for each generated image, we use a  pre-trained classification model which is the SOTAs model in the community to evaluate its probability of belonging to each category. If the image has probabilities greater than a given threshold $\lambda$ on more than two categories, it is identified with mode mixture. Given the classification model $cls(\boldsymbol{y})$=$(p_1^{\boldsymbol{y}}, p_2^{\boldsymbol{y}}, ..., p_C^{\boldsymbol{y}})$, number of categories $C$, the definition of \textbf{mode mixture indicator} is:
\begin{equation}
\label{eq:indicative}
     \mathbf{I}(\boldsymbol{y}) =  \left \{
     \begin{array}{ll}
               1, & |\{p_j^{\boldsymbol{y}} | p_j^{\boldsymbol{y}} \geq \lambda \}| \geq 2 \\
               0, & others,
     \end{array}
                  \right.j=1,\cdots,C,
\end{equation}
where $|\cdot|$ denotes the number of elements in the set,  $p_j^{\boldsymbol{y}}$ represents the probability that the image $\boldsymbol{y}$ belongs to the $j$-th category.  Then, We use mode mixture ratio (\textbf{MMR}) to measure the performance of the generative model in avoiding mode mixture, which is defined as follows:
\begin{equation}
\label{eq:MM}
    \textbf{MMR} = \frac{1}{K}\sum_{i=1}^{K} \mathbf{I}(\boldsymbol{y}_i),~~i=1,\cdots,K,
\end{equation}
where $K$ is the number of images. $\mathbf{I}(\boldsymbol{y}_i)$ indicates the prediction of whether there is a mode mixture on $i$-th image. $\textbf{MMR}$ is a simple but effective metric, which circumvents the dilemma of accuracy calculations but no labels. The lower its value, the fewer images with mode mixture.

\subsection{Sample Quality and Efficiency}
To illustrate the remarkable performance of the \textbf{DPM-OT}, we give its analysis of sampling quality and efficiency. As the universal evaluation metrics of image generation, the FID scores for our model on CIFAR-10, CelebA, and FFHQ datasets are reported in Tab. \ref{tab:FID}.  Moreover, in order to evaluate the quality of the generated models in several dimensions, we give the results of precision and recall in Tab. \ref{tab:PR} on CelebA. Precision is quantified by querying for each generated image whether the image is within the estimated manifold of real images. Symmetrically, recall is calculated by querying for each real image whether the image is within the estimated manifold of the generated image. 

Combining the above metrics, the following result analysis is provided. Results on CIFAR-10, the FID of our model outperforms other fast DPM methods, this shows that our model is able to improve the image quality while accelerating sampling speed. And compared with DDIM, NCSNv2, EDM, and DPM-solver, the images generated by our method are much sharper, while other models have some images with mode mixture. Such as a bird head with a horse body, as marked in the red box of Fig. \ref{figure:cifar}. 
Results on celebA, Tab. \ref{tab:FID} and Fig. \ref{figure:celeba} provide quality evaluation and visualization comparisons in CelebA, respectively. Our model only needs $5$ function evaluaitons to achieve excellent performance in image quality assessment and compared to the other models in Tab. \ref{tab:FID}, the \textbf{DPM-OT} is superior in terms of both speed and quality. Although DDIM, NCSNv2, and DPM-solver get highly realistic images,
their results appear to some not exist face images in the real world caused by mode mixture, such as the red box marks in Fig. \ref{figure:celeba}, while our method does not arise in this case.
Results on FFHQ, in comparison to DPM-solver and UniPC, our method attains superior results after 5 times of neural network inference. Fig. \ref{figure:celeba} shows the \textbf{DPM-OT} generated high fidelity and almost no mixture images, yet the NCSNv2 appears some deformed face images in Fig. \ref{figure:celeba} with a marked red box. 
In addition to this, The results of precision and recall of our model on CelebA are shown in Tab. \ref{tab:PR}. Compared with NCSNv2, though the precision of our model is slightly lower, the recall is over a lot, which shows that our model has enhanced in both diversity and image quality.

In summary, the quantitative evaluation and visualized results demonstrate our model has superior performance. In particular, the sampling speed has been greatly enhanced, requiring only $5-10$ NFEs to generate high-quality images.

\subsection{Validity of Mode Mixture Indicator}\label{sec:Validity}
To verify the validity of the mode mixture indicator, i.e., Eq.~\eqref{eq:indicative}, we conduct empirical analysis and give the relevant examples and results which are reported in Fig.~\ref{figure:mixexample} and Fig.~\ref{figure:acc_thred}. We randomly select an image with mode mixture and an image with significant category characteristics to obtain classification results by the pre-training model ViT-B16~\cite{dosovitskiy2020}, judging whether the image is with mode mixture or not. Fig.~\ref{figure:mixexample} shows that the probability of a horse and deer in the mixed image is $0.3$ and $0.4$, respectively, this also aligns with our visual perception which is a mixture of a horse and a deer.  While the other image can immediately be recognized as an automobile, and the output probability of the classification model is close to $1$, this also indicates that the selected classification model is effective, and the indicator defined in Eq.~\eqref{eq:indicative} can effectively detect mode mixture.

Furthermore, we further randomly collected $100$ clear images from  CIFAR-10 and $50$ images with mode mixture from generated results, the accuracy of the indicator's results are shown in Fig.~\ref{figure:acc_thred}. We found that the mode mixture indicator has a great ability to recognize clear images. Moreover, when the threshold $\lambda$ is about $0.11$, the image with mode mixture can be effectively identified. In summary of the results, it is reasonable to assume the proposed indicator is effective in the detection of mode mixture.

\subsection{Mode Mixture Analysis}
This part will demonstrate that the \textbf{DPM-OT} can mitigate mode mixture from experimental results. We devised a new metric to quantitatively assess model mixture, i.e. Eq.~\eqref{eq:MM}, which assumes that no less than two components of the probability vector are greater than the setup threshold. Furthermore, we utilize the best pre-trained classification model ViT-B16 to count the number of blended images. The results of \textbf{MMR} are reported in Tab.~\ref{tab:modeMixture}, with a nearly consistent number of images, we calculate the ratio of obfuscated images under the corresponding threshold. According to the results in Section~\ref{sec:Validity}, we know that when the value of $\gamma$ is around $0.1$, the prediction of the mode mixture indicator on clear images and mixed images will be more accurate, so we have selected five values in the interval $[0.1,0.2]$ for comparative experiments. Judging from the results, the \textbf{DPM-OT} consistently maintains excellent performance under different thresholds, which indicates our model can effectively alleviate mode mixture.

In addition, the visualized results further illustrate the ability of our model to mitigate mode mixture in Fig.~\ref{figure:cifar} $\sim$ Fig.~\ref{figure:ffhq}. The visual results on the CIFAR-10 in Fig.~\ref{figure:cifar}, (a) $\sim$ (d) appear confusing images, such as a red box marked either a variety of animal images mixed together or blurred and unclear. And for face image results on the celebA $64 \times 64$ in Fig.~\ref{figure:celeba}, DDIM has a broken face formed by mixing faces of different sizes, NCSNv2 produces a completely deformed face, DPM-solver generates strange unnatural faces.
This also demonstrates that the \textbf{DPM-OT} can better mitigate mode mixture on face images.
The results of the high-resolution face image on FFHQ $256 \times 256$ are displayed in Fig. \ref{figure:ffhq}, Although NCSNv2 generates clean face images, there are still existed missing ears and overlapping facial features. The high-quality unconfused images are generated by our model. To sum up, from the \textbf{MMR} and visual results, the output results of other models have a serious mode mixture, and the \textbf{MMR} of our model maintains a very small prediction rate, this shows that our proposed method can well mitigate this problem and also improves the quality of images.  We then briefly analyze the reason for mode mixture via optimal transport theory.

According to the regularity theory of Monge-Amp$\acute{\textrm{e}}$re, if the support set of the target distribution is non-convex, then there exists a singular set of points with zero measure and the transport mapping is interrupted at the singular points. While traditional deep neural networks (DNN) can only approximate continuous maps, this internal conflict leads to mode collapse/mixture. 
We are aware that the transformation from noise to generated images involves a transition from continuous to discrete distribution, which results in singular boundaries, and the images generated at the boundary points inevitably appear mode mixture. While existing DPMs attempt to use DNN to approximate the continuous diffusion process, which arise intrinsic conflict and cause mode mixture. Our method calculates the SDOT map from $\boldsymbol{x}_T$ to $\boldsymbol{x}_M$ to avoid using DNNs approximate continuous map, which can eliminate mode mixture at the first step of \textbf{DPM-OT} sampler. Therefore, it can effectively mitigate mode mixture of our model.

\section{Conclusion}
In this paper, we propose a novel fast sampling diffusion probability model in combination with optimal transport, i.e. \textbf{DPM-OT}, which can generate high-quality images while greatly speeding up the sampling. Our method built an optimal trajectory from the prior distribution to the target latents distribution by calculating the SDOT map between them. The optimal trajectory provides a near-perfect initial value for the subsequent diffusion process through a single-step transmission, which greatly shortens the sampling trajectory, thus improving the sampling efficiency. Moreover, the discontinuity of the SDOT map at the boundary singular point dramatically alleviates the problem of mode mixture in the generated image.  Furthermore, the error bound of the proposed method is provided, which theoretically guarantees the stability of the algorithm. To detect mode mixture without labels, an effective indicator is proposed and verified. Extensive experiments validate the proposed \textbf{DPM-OT} can generate high-quality samples with almost no mode mixture within only $5-10$ function evaluations.

\paragraph{Limitations}
One limitation of our approach is that the noisy training data samples $\boldsymbol{x}_M$ need to be stored for use at sampling time. This means additional storage requirements, although we have designed batch-processing algorithms to reduce the demand for device storage. Another limitation is that we just only consider unconditional generation here. In future research, it would be interesting to incorporate the \textbf{DPM-OT} framework into conditional synthesis tasks.

{\small
\bibliographystyle{ieee_fullname}
\bibliography{ref}
}


\newpage
\appendix
\onecolumn
In this document, we provide proof of theorems, additional implementation
details, and qualitative results. Limitations and future work are also discussed.
\section{Proof of Theorems} \label{secA}
\begin{theorem}
	\cite{lei2020geometric,villani2008optimal} Given $\mu$ and $\nu$ on a compact convex domain $\Omega \subset \mathbb{R}^{d}$, there exists an OT plan for the cost $c(\boldsymbol{x}, \boldsymbol{y})=$ $\boldsymbol{g}(\boldsymbol{x}-\boldsymbol{y})$, with $\boldsymbol{g}$ strictly convex. It is unique and of the form (id, $\left.T_{\#}\right) \mu$ (id: identity map), provided that $\mu$ is absolutely continuous with respect to Lebesgue measure and $\partial \Omega$ is negligible. Moreover, there exists a Kantorovich's potential $\varphi$, and OT map $T$ can be represented as follows:
	\begin{equation*}
		T(\boldsymbol{x})=\boldsymbol{x}-(\nabla \boldsymbol{g})^{-1}[\nabla \varphi(\boldsymbol{x})].
		\label{eq:T}
	\end{equation*}
	\label{lemma1}
\end{theorem}

\begin{theorem}\label{theorem:oneerro}
	Let $\Tilde{\boldsymbol{x}}_t$ and $\boldsymbol{x}_t$ be the samples of step $t$ obtained by \textbf{DPM-OT} and forward diffusion respectively, and $t\leq M$, $\boldsymbol{\zeta}_M$ be the error at step $M$ induced by optimal trajectory, then there is a constant $C_t>0$ satisfies the following inequality.
	\begin{equation}
		\left\| \boldsymbol{\Tilde{x}}_t-\boldsymbol{x}_t \right\|\leq C_t \left\| \boldsymbol{\zeta}_M \right\|
	\end{equation} 
\end{theorem}
\begin{proof}
	Since the reverse diffusion function sequence $\left\{f_t\right\}$ is continuous, there is continuous function $c_t(\cdot)$ from $\mathbb{R}^{d}$ to $\mathbb{R}$ that makes the following formula hold
	\begin{equation}
		\begin{split}
			\Tilde{\boldsymbol{x}}_t&=f_{t}\circ\cdots\circ f_{M-1}(\boldsymbol{x}_M+\boldsymbol{\zeta}_M)\\
			&=f_{t}\circ\cdots\circ f_{M-1}(\boldsymbol{x}_M)+c_t(\boldsymbol{\zeta}_M)\boldsymbol{\zeta}_M\\
			&=\boldsymbol{x}_t+c_t(\boldsymbol{\zeta}_M)\boldsymbol{\zeta}_M
		\end{split}
	\end{equation}
	So we can get
	\begin{equation*}        
		\left\| \Tilde{\boldsymbol{x}}_t-\boldsymbol{x}_t \right\|
		=\left\|c_t(\boldsymbol{\zeta}_M)\boldsymbol{\zeta}_M\right\|
		\leq|c_t(\boldsymbol{\zeta}_M)|\left\|\boldsymbol{\zeta}_M\right\|
		\leq C_t\left\|\boldsymbol{\zeta}_M \right\|       
	\end{equation*}
\end{proof}

\begin{theorem}\label{theorem:errobound}
	Let $L_{dpm\_ot}$ be the error between the data distribution generated by \textbf{DPM-OT} and the target data distribution which is defined in Eq.~\ref{eq:erroDPMOT}, $L_{vlb}$ is the variational lower bound on negative log-likelihood between data distribution generated by vanilla DPM and the target data distribution which is defined in Eq.~\ref{eq:loss}.  We have $L_{dpm\_ot} \leqslant L_{vlb}$, i.e., $L_{vlb}$ is the upper bound of $L_{dpm\_ot}$.
	\begin{equation}\label{eq:erroDPMOT}
		\begin{split}
			L_{dpm\_ot}=&L_0 + L_1 + ... + L_M+L_{T}\\
			=&-\log \tilde{p}_{\theta}(\boldsymbol{x}_0 | \boldsymbol{x}_1)+D_{KL}(q(\boldsymbol{x}_T | \boldsymbol{x}_0)),p(\boldsymbol{x}_T))\\
			&+\sum_{t=1}^{M-1} D_{KL}(q(\boldsymbol{x}_t|\boldsymbol{x}_{t+1},\boldsymbol{x}_0)|| \tilde{p}_{\theta}(\boldsymbol{x}_t|\boldsymbol{x}_{t+1}))\\
			&+D_{KL}(q(\boldsymbol{x}_M|\boldsymbol{x}_T,\boldsymbol{x}_0)|| \tilde{p}_{\theta}(\boldsymbol{x}_M|\boldsymbol{x}_T)),
		\end{split}
	\end{equation}
	\begin{equation}
		\begin{aligned}
			L_{vlb} =& -\log p_{\theta}(\boldsymbol{x}_0 | \boldsymbol{x}_1)+  D_{KL}(q(\boldsymbol{x}_T | \boldsymbol{x}_0))||p(\boldsymbol{x}_T))  \\
			& +\sum_{t>1} D_{KL}(q(\boldsymbol{x}_{t-1}|\boldsymbol{x}_t,\boldsymbol{x}_0)|| p_{\theta}(\boldsymbol{x}_{t-1}|\boldsymbol{x}_t))\label{eq:loss} 
		\end{aligned}
	\end{equation}
\end{theorem}
\begin{proof}
	It may be assumed that$\boldsymbol{X}_t = \boldsymbol{x}_t-\mu_{t+1}(\boldsymbol{x}_{t+1})$,
	\begin{equation}
		\begin{aligned}
			p_{\theta}(\boldsymbol{x}_{t}|\boldsymbol{x}_{t+1}) = \mathcal{N}(\boldsymbol{x}_{t}|\mu_{t+1}(\boldsymbol{x}_{t+1}), \sigma_{t+1}^2I)\\
			=\frac{1}{(2\pi)^{d/2}\sigma_{t+1}^{d}} 
			\exp\left({-\frac{1}{2\sigma_{t+1}^2}}\boldsymbol{X}_t^{T}\boldsymbol{X}_t\right)
		\end{aligned}
	\end{equation}
	\begin{equation}
		\begin{aligned}
			p_{\theta}(\tilde{\boldsymbol{x}}_{t}|\boldsymbol{x}_{t+1})=\mathcal{N}(\boldsymbol{x}_{t}+\boldsymbol{\zeta}_{t}| \boldsymbol{\mu}_{t+1}(\boldsymbol{x}_{t+1}), \sigma_{t+1}^2I)\\
			=\frac{1}{(2\pi)^{d/2}\sigma_{t+1}^{d}}  
			\exp\left(-\frac{1}{2\sigma_{t+1}^2}\left(\boldsymbol{X}_t-\boldsymbol{\zeta}_t)^{T}(\boldsymbol{X}_t-\boldsymbol{\zeta}_t\right)\right)
		\end{aligned}
	\end{equation}
	By definitions of $p_{\theta}(\boldsymbol{x}_t|\boldsymbol{x}_{t+1})$ and $p_{\theta}(\tilde{\boldsymbol{x}}_t|\boldsymbol{x}_{t+1})$, we know the following equation
	\begin{equation}\label{sup-1}
		\begin{split}
			&\left | \log\frac{p_{\theta}(\boldsymbol{x}_t|\boldsymbol{x}_{t+1})}{p_{\theta}(\tilde{\boldsymbol{x}}_t|\boldsymbol{x}_{t+1})} \right |\\
			=&\frac{1}{2\sigma_{t+1}^2}\left|2\boldsymbol{x}_t\boldsymbol{\zeta}_t+\boldsymbol{\zeta}_t^T\boldsymbol{\zeta}_t-2\boldsymbol{\mu}_{t+1}(\boldsymbol{x}_{t+1})^T\boldsymbol{\zeta}_t\right|\\
			\leq&\frac{1}{2\sigma_{t+1}^2}\left ( 2||\boldsymbol{x}_t||\cdot||\boldsymbol{\zeta}_t||+||\boldsymbol{\zeta}_t||^2+2||\boldsymbol{\mu}_{t+1}(\boldsymbol{x}_{t+1})||\cdot\parallel\boldsymbol{\zeta}_t\parallel \right )\\
			=&\frac{1}{2\sigma_{t+1}^2}\left(2||\boldsymbol{x}_t||+||\boldsymbol{\zeta}_t||+2||\boldsymbol{\mu}_{t+1}(\boldsymbol{x}_{t+1})||\right )||\boldsymbol{\zeta}_t||
		\end{split}
	\end{equation}
	Suppose $\boldsymbol{x}$ is bounded with $[a, b]^d$, then there is constant $A_t>0$ makes
	\begin{equation}\label{sup-2}
		2||\boldsymbol{x}_t||+||\boldsymbol{\zeta}_t||+2||\boldsymbol{\mu}_{t+1}(\boldsymbol{x}_{t+1})||\leq A_t
	\end{equation}
	Applying inequality \ref{sup-2} to equation \ref{sup-1}, we get
	\begin{equation}
		\begin{split}   
			|\tilde{D}_{KL}- D_{KL}|
			=&\left| \int_{X_t}q(\boldsymbol{x}_t|\boldsymbol{x}_{t+1}) \log\frac{p_{\theta}(\boldsymbol{x}_t|\boldsymbol{x}_{t+1})}{p_{\theta}(\tilde{\boldsymbol{x}}_t|\boldsymbol{x}_{t+1})}d\boldsymbol{x}_t \right|\\
			\leq& \frac{A_t}{2\sigma_{t+1}^2}\cdot||\boldsymbol{\zeta}_t||
		\end{split}
	\end{equation}
	where $\tilde{D}_{KL} = D_{KL}(q(\boldsymbol{x}_t|\boldsymbol{x}_{t+1})|| p_{\theta}(\tilde{\boldsymbol{x}}_t|\boldsymbol{x}_{t+1}))$,  $D_{KL} = D_{KL}(q(\boldsymbol{x}_t|\boldsymbol{x}_{t+1})|| p_{\theta}(\boldsymbol{x}_t|\boldsymbol{x}_{t+1}))$. Therefore there is
	\begin{equation}
		\begin{split}
			&|L_{dmot}-L_{vlb}^{0:M}|\\
			\leq& 
			\sum_{t=0}^{M}|\tilde{D}_{KL}- D_{KL}|
			\\
			\leq&(M+1)\max_t \frac{A_t}{2\sigma_{t+1}^2} |C_t|\cdot||\boldsymbol{\zeta}_M||
		\end{split}
	\end{equation}
	That is
	\begin{equation}
		L_{dmot}\leq L_{vlb}^{0:M}+(M+1)\max_t \frac{A_t}{2\sigma_{t+1}^2} |C_t|\cdot||\boldsymbol{\zeta}_M||
	\end{equation}
	Because of $\boldsymbol{\zeta}_M=O(N^{-\frac{1}{2}})$, we can make $\boldsymbol{\zeta}_M$ arbitrarily small by increasing the number of $OT$ samples $N$. For a given \textbf{DPM}, there exists $\boldsymbol{\zeta}_M$ such that the following formula is true.
	\begin{equation}
		(M+1)\max_t \frac{A_t}{2\sigma_{t+1}^2} |C_t|\cdot||\boldsymbol{\zeta}_M||\leq L_{vlb}^{M+1:T}
	\end{equation}
	So we have
	\begin{equation}
		L_{dmot}\leq L_{vlb}
	\end{equation}
	
\end{proof}

\section{Implementation Details}
Additional details about hyperparameter settings of \textbf{Algorithm 1} and \textbf{Algorithm 2} are elucidated in this section.  In the experiments, we instantiate the DPM model $\mathbf{s}_\theta$ of \textbf{DPM-OT} with pre-trained models of NCSNv2 \cite{NEURIPS2020_92c3b916}. In addition, for a fair comparison, we also adopted the same sampling schedule $\{(b_t,\sigma_t)\}_{t=0}^{T}$ as NCSNv2.

In \textbf{Algorithm 1}, we use the Monte Carlo method to solve the SDOT map. We set the number of Monte Carlo samples $N=10\times |\mathcal{I}|$, where $|\cdot|$ denotes the number of elements in the set. For the learning rate $lr$, we set the $lr$ on the datasets CIFAR10~\cite{krizhevsky2009learning}, CelebA~\cite{liu2015deep} and FFHQ~\cite{Karras_2019_CVPR} to $lr=0.1$, $lr=20$, and $lr=50$, respectively. For better convergence, we double the number of samples $N$ and multiply the learning rate $lr$ by 0.8 when the energy function $E(\boldsymbol{h})$ has not decreased for $s=50$ steps.  Moreover, we set threshold $\tau=8\times 10^{-4}$. When the energy function $E(\boldsymbol{h})<\tau$ or the total number of iteration steps is greater than 10000, the optimization of $\boldsymbol{h}$ will be stopped.

In \textbf{Algorithm 2}, reverse diffusion steps $M$ is a variable that satisfies $0<M<T$. We have carried out five experiments on the datasets CIFAR10, CelebA, and FFHQ, where values of $M$ are 5, 10, 20, 30 and 50 respectively. We find that with the increase of $M$, the image FID score will decrease, but this decline will tend to be flat with the increase of $M$, which is reflected in Table 1 of the paper.



\section{Additional Qualitative Results}
\begin{figure}[ht]
	\centering
	\subfigure{
		\includegraphics[width=0.49\textwidth]{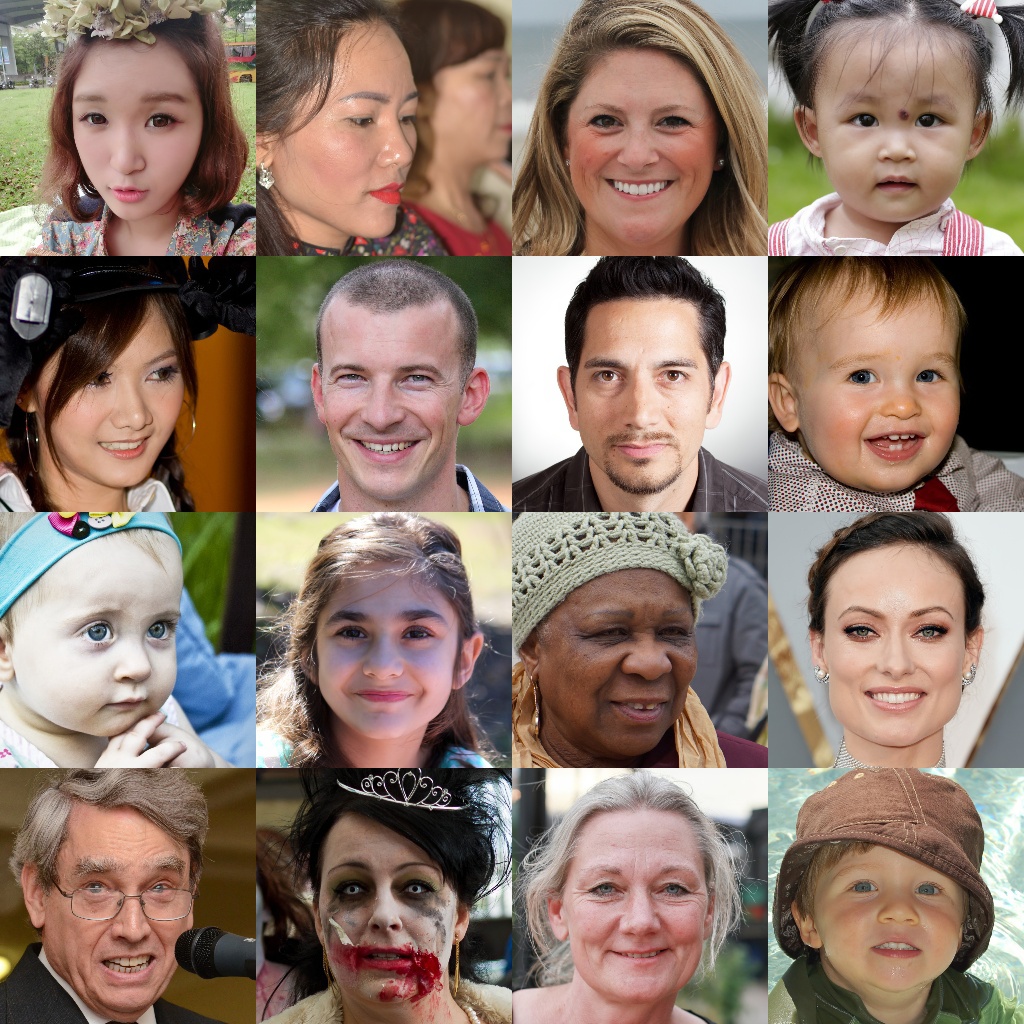}}
	\subfigure{
		\includegraphics[width=0.49\textwidth]{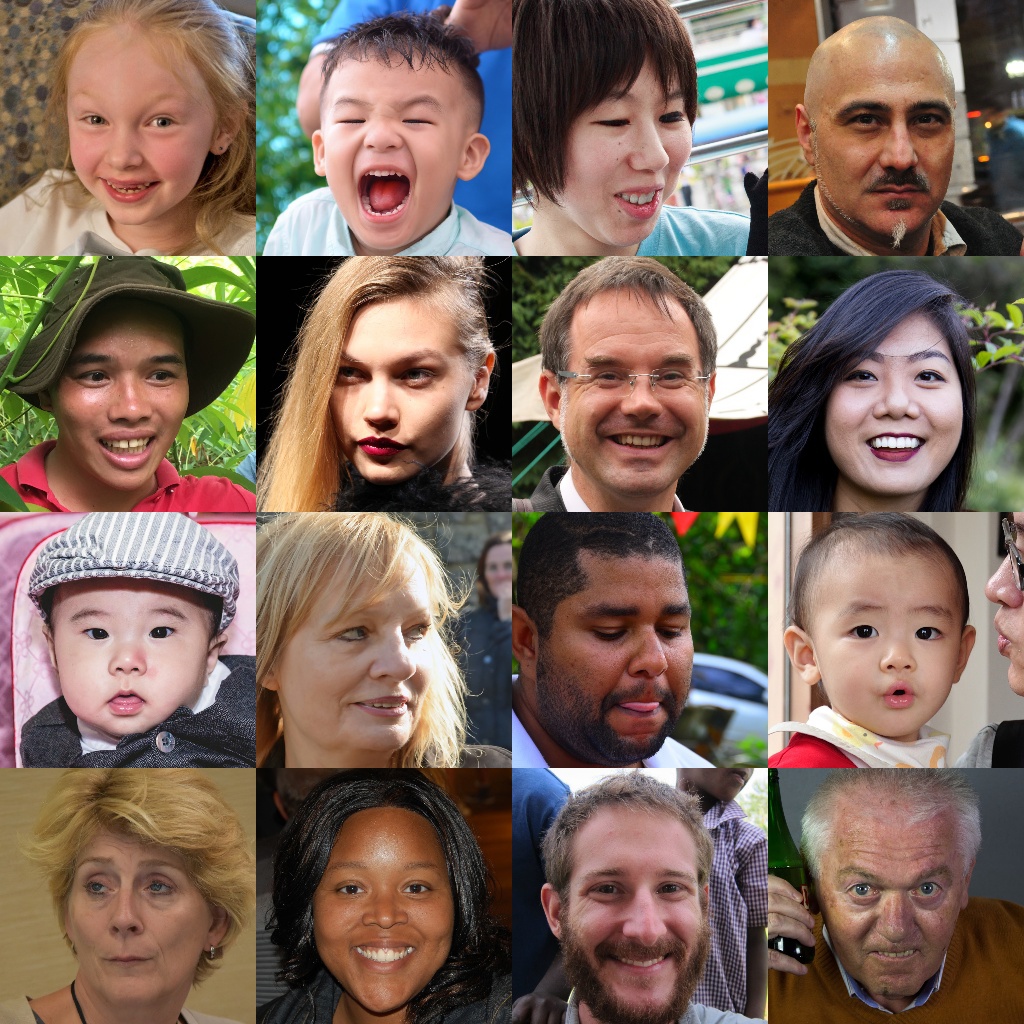}}
	\setlength{\abovecaptionskip}{0cm}  
	\setlength{\belowcaptionskip}{0cm} 
	\caption{The visualization of our model on FFHQ (10 steps).}
	\label{FFHQ}
\end{figure}

In this section, we show three more qualitative results of the proposed \textbf{DPM-OT} on FFHQ, Cifar10, and CelebA respectively.
Among them, Fig.~\ref{FFHQ} and Fig.~\ref{Cifar10} show the results obtained by our model with 10 steps of inverse diffusion on the FFHQ $256\times 256$ and Cifar10, respectively, and Fig.~\ref{CelebA} shows the results obtained with only 5 steps of inverse diffusion on the CelebA dataset. As these results show, our model can obtain high-quality images after 5-10 reverse diffusion.

\begin{figure}[ht]
	\centering
	\includegraphics[width=0.5\textwidth]{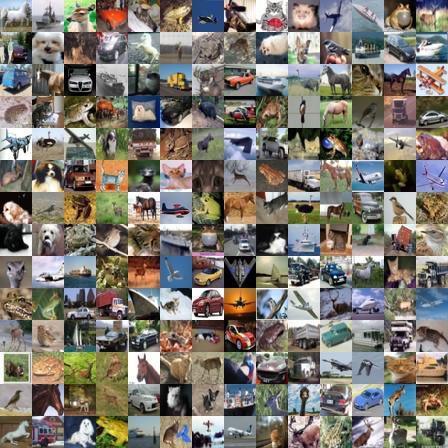}
	\setlength{\abovecaptionskip}{0cm}  
	\setlength{\belowcaptionskip}{0cm} 
	\caption{The visualization of our model on Cifar10 (10 steps).}
	\label{Cifar10}
\end{figure}

\begin{figure}[ht]
	\centering
	\includegraphics[width=0.5\textwidth]{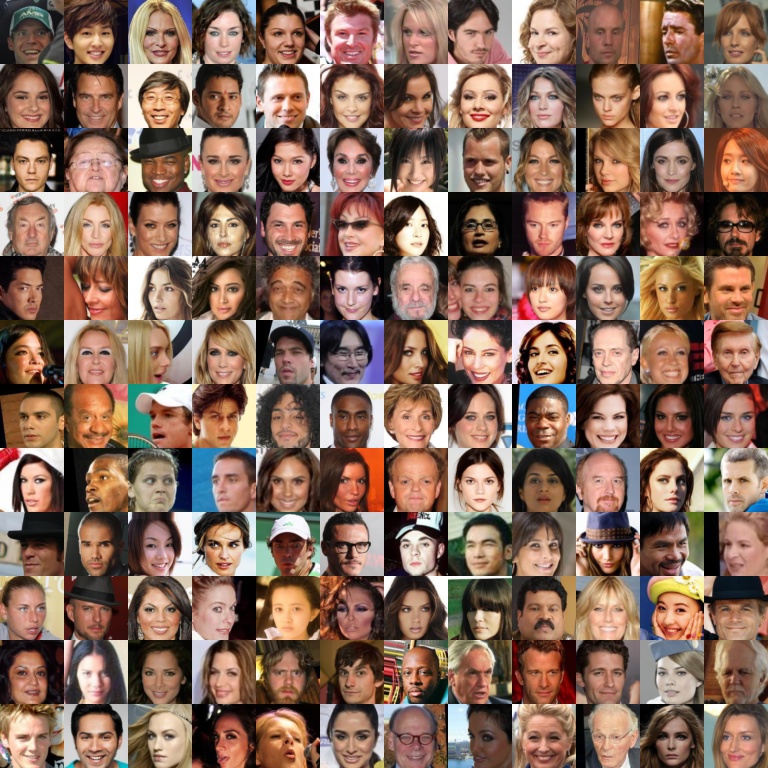}
	\setlength{\abovecaptionskip}{0cm}  
	\setlength{\belowcaptionskip}{0cm} 
	\caption{The visualization of our model on CelebA (5 steps).}
	\label{CelebA}
\end{figure}

\end{document}